%% file: llncs_main.tex
\documentclass[runningheads]{llncs}

\usepackage{cite}
\usepackage{amsmath}
\allowdisplaybreaks
\usepackage{amsfonts}
\usepackage{amssymb}
\usepackage{array}
\usepackage{subfig}
\usepackage{url}
\usepackage{graphicx}
\usepackage{hyperref}
\usepackage{algorithm}
\usepackage{algpseudocode}

\usepackage{booktabs}
\usepackage{color}
\usepackage{ulem}
\usepackage[font=small,labelfont=bf]{caption}
\newtheorem{assumption}{Assumption}
\usepackage[misc,geometry]{ifsym}

\newcommand{\NAME}{\textsc{FedGSP}}
\newcommand{\NAIVE}{NaiveGSP}
\newcommand{\GROUPER}{ICG}

\begin{document}

\title{Heterogeneous Federated Learning via Grouped Sequential-to-Parallel Training}
\titlerunning{Federated Grouped Sequential-to-Parallel Training}

\author{Shenglai Zeng\inst{1}\thanks{Shenglai Zeng and Zonghang Li are of equal contributions.} \and
Zonghang Li\inst{1,3}$^\star$ \and
Hongfang Yu\inst{1}(\Letter) \and
Yihong He\inst{1} \and \\
Zenglin Xu\inst{2} \and
Dusit Niyato\inst{3} \and
Han Yu\inst{3}}
\authorrunning{S. Zeng, Z. Li, H. Yu, et al.}

\institute{Sch of Info \& Comm Engin, University of Electronic Science and Technology of China, Chengdu, China \\
\email{\{shenglaizeng, lizhuestc, heyh.uestc\}@gmail.com, yuhf@uestc.edu.cn} \and
Sch of Comp Sci \& Tech, Harbin Institute of Technology, Shenzhen, China \\
\email{xuzenglin@hit.edu.cn} \and
Sch of Comp Sci \& Engin, Nanyang Technological University, Singapore \\
\email{\{dniyato, han.yu\}@ntu.edu.sg}}

\maketitle

\input{abstract}

\input{introduction}
\input{related-work}
\input{method}
\input{experiment}
\input{conclusion}

\vspace{5mm}
\noindent \textbf{Acknowledgments.} This work is supported, in part, by the National Key Research and Development Program of China (2019YFB1802800); PCL Future Greater-Bay Area Network Facilities for Large-Scale Experiments and Applications (LZC0019), China; National Research Foundation, Singapore under its AI Singapore Programme (AISG Award No: AISG2-RP-2020-019); the RIE 2020 Advanced Manufacturing and Engineering (AME) Programmatic Fund (No. A20G8b0102), Singapore; and Nanyang Assistant Professorship (NAP). Any opinions, findings and conclusions or recommendations expressed in this material are those of the authors and do not reflect the views of the funding agencies.


\input{ref}
\end{document}

%% file: abstract.tex
\begin{abstract}
Federated learning (FL) is a rapidly growing privacy preserving collaborative machine learning paradigm. In practical FL applications, local data from each data silo reflect local usage patterns. Therefore, there exists heterogeneity of data distributions among data owners (a.k.a. FL clients). If not handled properly, this can lead to model performance degradation. This challenge has inspired the research field of heterogeneous federated learning, which currently remains open. In this paper, we propose a data heterogeneity-robust FL approach, \NAME, to address this challenge by leveraging on a novel concept of dynamic Sequential-to-Parallel (STP) collaborative training. \NAME~assigns FL clients to homogeneous groups to minimize the overall distribution divergence among groups, and increases the degree of parallelism by reassigning more groups in each round. It is also incorporated with a novel Inter-Cluster Grouping (\GROUPER) algorithm to assist in group assignment, which uses the centroid equivalence theorem to simplify the NP-hard grouping problem to make it solvable. Extensive experiments have been conducted on the non-i.i.d. FEMNIST dataset. The results show that \NAME~improves the accuracy by 3.7\% on average compared with seven state-of-the-art approaches, and reduces the training time and communication overhead by more than 90\%.

\begin{keywords}
Federated learning \and 
Distributed data mining \and
Heterogeneous data \and  Clustering-based learning
\end{keywords}
\end{abstract}

%% file: introduction.tex
\section{Introduction}
Federated learning (FL) \cite{kairouz2019advances}, as a privacy-preserving collaborative paradigm for training machine learning (ML) models with data scattered across a large number of data owners, has attracted increasing attention from both academia and industry. Under FL, data owners (a.k.a. FL clients) submit their local ML models to the FL server for aggregation, while local data remain private. FL has been applied in fields which are highly sensitive to data privacy, including healthcare \cite{xu2021federated}, manufacturing \cite{khan2021federated} and next generation communication networks \cite{lim2020federated}. In practical applications, FL clients' local data distributions can be highly heterogeneous due to diverse usage patterns. This problem is referred to as the non-independent and identically distributed (non-i.i.d.) data challenge, which negatively affects training convergence and the performance of the resulting FL model\cite{zhao2018federated}.

Recently, heterogeneous federated learning approaches have been proposed in an attempt to address this challenge. These works try to make class distributions of different FL clients similar to improve the performance of the resulting FL model. In \cite{zhao2018federated, yao2019federated, yoshida2020hybridfl}, FL clients share a small portion of local data to build a common meta dataset to help correct deviations caused by non-i.i.d. data. In \cite{jeong2018communication, duan2019astraea}, data augmentation is performed for categories with fewer samples to reduce the skew of local datasets. These methods are vulnerable to privacy attacks as misbehaving FL servers or clients can easily compromise the shared private data and the augmentation process. To align client data distributions without exposing the FL process to privacy risks, we group together heterogeneous FL clients so that each group can be perceived as a homogeneous ``client'' to participate in FL. This process does not involve any manipulation of private data itself and is therefore more secure.

An intuitive approach to achieve this goal is to assign FL clients to groups with similar overall class distribution, and use collaborative training to coordinate model training within and among groups. However, designing such an approach is not trivial due to the following two challenges. Firstly, assigning FL clients to a specified number of groups of equal group sizes to minimize the data divergence among groups (which can be reduced from the well-known bin packing problem \cite{garey1978strong}) is an NP-hard problem. Moreover, such group assignment process needs to be performed periodically in a dynamic FL environment, which introduces higher requirements for its effectiveness and execution efficiency. Secondly, even if the data distributions among groups are forced to be homogeneous, the data within each group can still be skewed. Due to the robustness of sequential training mode (STM) to data heterogeneity, some collaborative training approaches (e.g., \cite{duan2019astraea}) adopt STM within a group to train on skewed client data. Then, the typical parallel training mode (PTM) can be applied among homogeneous groups. These methods are promising, but are still limited due to their static properties, which prevents them from adapting to the changing needs of FL at different stages. In FL, STM should be emphasized in the early stage to achieve a rapid increase in accuracy in the presence of non-i.i.d. data, while PTM should be emphasized in the later stage to promote convergence. In the static mode, the above parallelism degree must be carefully designed to realize a proper trade-off between sensitivity to heterogeneous data of PTM and overfitting of STM. Otherwise, the FL model performance may suffer.

To address these challenges, this paper proposes a new concept of dynamic collaborative Sequential-to-Parallel (STP) training to improve FL model performance in the presence of non-i.i.d. data. The core idea of STP is to force STM to be gradually transformed into PTM as FL model training progresses. In this way, STP can better refine unbiased model knowledge in the early stage, and promote convergence while avoiding overfitting in the later stage. To support the proposed STP, we propose a \underline{Fed}erated \underline{G}rouped \underline{S}equential-to-\underline{P}arallel (\NAME) training framework. \NAME~allows reassignment of FL clients into more groups in each training round, and introduces group managers to manage the dynamically growing number of groups. It also coordinates model training and transmission within and among groups. In addition, we propose a novel Inter-Cluster Grouping (\GROUPER) method to assign FL clients to a pre-specified number of groups, which uses the centroid equivalence theorem to simplify the original NP-hard grouping problem into a solvable constrained clustering problem with equal group size constraint. \GROUPER~can find an effective solution with high efficiency (with a time complexity of $\mathcal{O}(\frac{K^6\mathcal{F}\tau}{M^2}\log{Kd})$). We evaluate \NAME~on the most widely adopted non-i.i.d. benchmark dataset FEMNIST\cite{caldas2018leaf} and compare it with seven state-of-the-art approaches including FedProx\cite{li2018federated}, FedMMD\cite{yao2018two}, FedFusion\cite{yao2019towards}, IDA\cite{yeganeh2020inverse}, FedAdam, FedAdagrad and FedYogi\cite{reddi2020adaptive}. The results show that \NAME~improves model accuracy by 3.7\% on average, and reduces training time and communication overhead by more than 90\%. To the best of our knowledge, \NAME~is the first dynamic collaborative training approach for FL.


%% file: related-work.tex
\section{Related Work}
Existing heterogeneous FL solutions can be divided into three main categories: 1) data augmentation, 2) clustering-based learning, and 3) adaptive optimization.

\textbf{Data Augmentation:} Zhao et al. \cite{zhao2018federated} proved that the FL model accuracy degradation due to heterogeneous local data can be quantified by the earth move distance (EMD) between the client and global data distributions. This result motivates some research works to balance the sample size of each class through data augmentation. Zhao et al. \cite{zhao2018federated} proposed to build a globally shared dataset to expand client data. Jeong et al. \cite{jeong2018communication} used the conditional generative network to generate new samples for categories with fewer samples. Similarly, Duan et al. \cite{duan2019astraea} used augmentation techniques such as random cropping and rotation to expand client data. These methods are effective in improving the FL model accuracy by reducing data skew. However, they involve modifying clients' local data, which can lead to serious privacy risks.

\textbf{Clustering-based Learning:} Another promising way to reduce data heterogeneity is through clustering-based FL. Sattler et al. \cite{sattler2020clustered} groups FL clients with similar class distributions into one cluster, so that FL clients with dissimilar data distributions do not interfere with each other. This method works well in personalized FL \cite{fallah2020personalized} where FL is perform within each cluster and an FL model is produced for each cluster. However, it is not the same as our goal which is to train one shared FL model that can be generalized to all FL clients. Duan et al. \cite{duan2019astraea} makes the KullbackLeibler divergence of class distributions similar among clusters, and proposed a greedy best-fit strategy to assign FL clients.

\textbf{Adaptive Optimization:} Other research explores adaptive methods to better merge and optimize client- and server-side models. On the client side, Li et al. \cite{li2018federated} added a proximal penalty term to the local loss function to constrain the local model to be closer to the global model. Yao et al. \cite{yao2018two} adopted a two-stream framework and used transfer learning to transfer knowledge from the global model to the local model. A feature fusion method has been further proposed to better merge the features of local and global models \cite{yao2019towards}. On the server side, Yeganeh et al. \cite{yeganeh2020inverse} weighed less out-of-distribution models based on inverse distance coefficients during aggregation. Instead, Reddi et al. \cite{reddi2020adaptive} focused on server-side optimization and introduced three advanced adaptive optimizers (Adagrad, Adam and Yogi) to obtain FedAdagrad, FedAdam and FedYogi, respectively. These methods perform well in improving FL model convergence.

Solutions based on data augmentation are at risky due to potential data leakage, while solutions based on adaptive optimization do not solve the problem of class distribution divergence causing FL model performance to degrade. \NAME~focuses on clustering-based learning. Different from existing research, it takes a novel approach of dynamic collaborative training, which allows dynamic scheduling and reassignment of clients into groups according to the changing needs of FL.

%% file: method.tex
\section{Federated Grouped Sequential-to-Parallel Learning}
In this section, we first describe the concept and design of the STP approach. Then, we present the \NAME~framework which is used to support STP. Finally, we mathematically formulate the group assignment problem in STP, and present our practical solution \GROUPER.

\subsection{STP: The Sequential-to-Parallel Training Mode}\label{section:stp}
Under our grouped FL setting, FL clients are grouped such that clients in the same group have heterogeneous data but the overall data distributions among the groups are homogeneous. Due to the difference in data heterogeneity, the training modes within and among groups are designed separately. We refer to this jointly designed FL training mode as the ``collaborative training mode''.

\begin{figure}[ht]
\centering
\includegraphics[width=\textwidth]{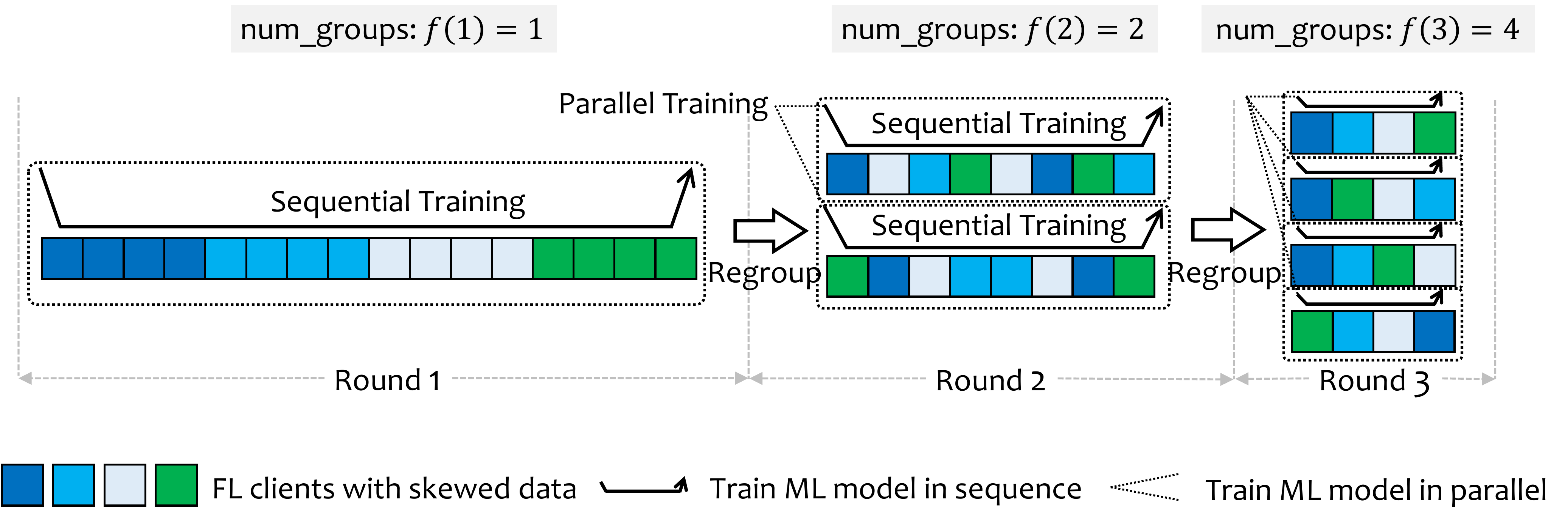}
\caption{An example of STP. The ML model in each group is trained in sequence, while ML models among groups are trained in parallel. In each round $r$, the group number grows according to function $f$, and FL clients are regrouped and shuffled.}
\label{fig:sequential-to-parallel}
\end{figure}

Intuitively, the homogeneous groups can be trained in a simple parallel mode PTM because the heterogeneity of their data has been eliminated by client grouping. Instead, for FL clients in the same group whose local data are still skewed, the sequential mode STM can be useful. In STM, FL clients train the model in a sequential manner. A FL client receives the model from its predecessor client and delivers the local trained model to its successor client to continue training. In the special case of training with only one local epoch (i.e., $e=1$), STM is equivalent to centralized SGD, which gives it robustness against data heterogeneity.

This naive collaborative training mode is static and has limitations. Therefore, we extend it to propose a more dynamic approach STP. As shown in Figure \ref{fig:sequential-to-parallel}, STP reassigns FL clients into $f(r)$ groups and shuffles their order in each round $r$, where $f$ is a pre-specified group number growth function, with the goal to dynamically adjust the degree of parallelism. Then, STP can be smoothly transformed from (full) sequential mode to (full) parallel mode. This design can prevent catastrophic forgetting caused by the long ``chain of clients'' that causes the FL model to forget the data of previous clients and overfit the data of subsequent clients, and can also prevent the FL model from learning interfering information such as the order of clients. Moreover, the growing number of groups improves the parallelism efficiency, which promotes convergence and speeds up training when the global FL model is close to convergence.

\begin{algorithm}[t]
\caption{\colorbox{green}{\textsc{\textbf{Sequential-To-Parallel}}} \textbf{(main)}}\label{alg:stp}
\begin{algorithmic}[1]
\Require All FL clients $\mathcal{C}$, the total number of FL clients $K$, the maximum training rounds $R$, the group number growth function $f$, the group sampling rate $\kappa$.
 \Ensure The well-trained global FL model $\omega_{\mathrm{global}}^{R}$.
\vspace{1mm}
\State Initialize the global FL model $\omega_{\mathrm{global}}^{0}$;
\For{each round $r=1,\cdots,R$}
\State Reassign all FL clients $\mathcal{C}$ to $f(r)$ groups to obtain $\mathcal{G}$,\label{code:tab1}\\ 
\qquad\qquad\qquad\qquad $\mathcal{G}\gets$\Call{\colorbox{yellow}{\textbf{Inter-Cluster-Grouping}}}{$\mathcal{C}, K, f, r$};
\State Randomly sample a subset of groups $\tilde{\mathcal{G}}\subset\mathcal{G}$ with proportion $\kappa$;\label{code:tab2}
\For{each group $\mathcal{G}_{m}$ \textbf{in} $\tilde{\mathcal{G}}$ in parallel}
\State The first FL client $\mathcal{C}_m^1$ in $\mathcal{G}_m$ initializes $\omega_m^1\gets\omega_{\mathrm{global}}^{r-1}$;\label{code:tab3}
\For{each FL client $\mathcal{C}_m^k$ \textbf{in} $\mathcal{G}_m$ in sequence}
\State Train $\omega_m^k$ on local data $\mathcal{D}_m^k$ using mini-batch SGD for one epoch;\label{code:tab4}
\State Send the trained $\omega_m^{k+1}\gets\omega_m^k$ to the next FL client $\mathcal{C}_m^{k+1}$;\label{code:tab5}
\EndFor
\State The last FL client $\mathcal{C}_m^{K/f(r)}$ in $\mathcal{G}_m$ uploads $\omega_m^{K/f(r)}$;\label{code:tab6}
\EndFor
\State Update the global FL model using the aggregation $\omega_{\mathrm{global}}^{r}\gets\frac{\sum_{\forall \mathcal{G}_m\in\tilde{\mathcal{G}}}{(\omega_m^{K/f(r)}})}{f(r)}$;\label{code:tab7}
\EndFor
\State \Return $\omega_{\mathrm{global}}^{R}$;\label{code:tab8}
\end{algorithmic}
\end{algorithm}

The pseudo code of STP is given in Algorithm \ref{alg:stp}. In round $r$, STP divides all FL clients into $f(r)$ groups using the \GROUPER~grouping algorithm (Line \ref{code:tab1}), which will be described in Section \ref{section:grouping}. Due to the similarity of data among groups, each group can independently represent the global distribution, so only a small proportion of $\kappa$ groups are required to participate in each round of training (Line \ref{code:tab2}). The first FL client in each group pulls the global model from the FL server (Line \ref{code:tab3}), and trains its local model using mini-batch SGD for one epoch (Line \ref{code:tab4}). The trained local model is then delivered to the next FL client to continue training (Line \ref{code:tab5}), until the last FL client is reached. The last FL client in each group sends the trained model to the FL server (Line \ref{code:tab6}). Models from all groups are aggregated to update the global FL model (Line \ref{code:tab7}). The above steps repeat until the maximum training round $R$ is reached. Finally, the well-trained global FL model is obtained (Line \ref{code:tab8}).

The choice of the growth function for the number of groups, $f$, is critical for the performance of STP. We give three representative growth functions, including linear (smooth grow), logarithmic (fast first and slow later), and exponential (slow first and fast later) growth functions:
\begin{align}
& \mathrm{Linear\ Growth\ Function:} & f(r)=&\beta\left\lfloor\alpha (r-1)+1\right\rfloor, \\
& \mathrm{Log\ Growth\ Function:} & f(r)=&\beta\left\lfloor\alpha\ln r+1\right\rfloor, \\
& \mathrm{Exp\ Growth\ Function:} & f(r)=&\beta\lfloor(1+\alpha)^{r-1}\rfloor,
\end{align}
where the real number coefficient $\alpha$ controls the growth rate, and the integer coefficient $\beta$ controls the initial number of groups and the growth span. We recommend to initialize $\alpha$, $\beta$ to a moderate value and explore the best setting in an empirical manner.

\subsection{\NAME: The Grouped FL Framework To Enable STP}\label{section:fedgsp}
In this section, we describe the \NAME~framework that enables dynamic STP. \NAME~is generally a grouped FL framework that supports dynamic group management, as shown in Figure \ref{fig:framework}. The basic components include a top server (which acts as an FL server and performs functions related to group assignment) and a large number of FL clients. FL clients can be smart devices with certain available computing and communication capabilities, such as smart phones, laptops, mobile robots and drones. They collect data from the surrounding environment and use the data to train local ML models.

In addition, \NAME~creates group managers to facilitate the management of the growing number of groups in STP. The group managers can be virtual function nodes deployed in the same machine as the top server. Whenever a new group is built, a new group manager is created to assist the top server to manage this group by performing the following tasks:

\textbf{1. Collect distribution information.} The group manager needs to collect class distributions of FL clients and report them to the top server. These meta information will be used to assign FL clients to $f(r)$ groups via \GROUPER. 

\textbf{2. Coordinate model training.} The group manager needs to coordinate the sequential training of FL clients in its group, as well as the parallel training with other groups, according to the rules of STP. Specifically, it needs to shuffle the order of clients and report resulting model to the top server for aggregation.

\textbf{3. Schedule model transmission.} In applications such as Industrial IoT systems, wireless devices can directly communicate with each other through wireless sensor networks (WSNs). However, this cannot be realized in most scenarios. Therefore, the group manager needs to act as a communication relay to schedule the transmission of ML models from one client to another.

\begin{figure}[t]
\centering
\includegraphics[width=.9\textwidth]{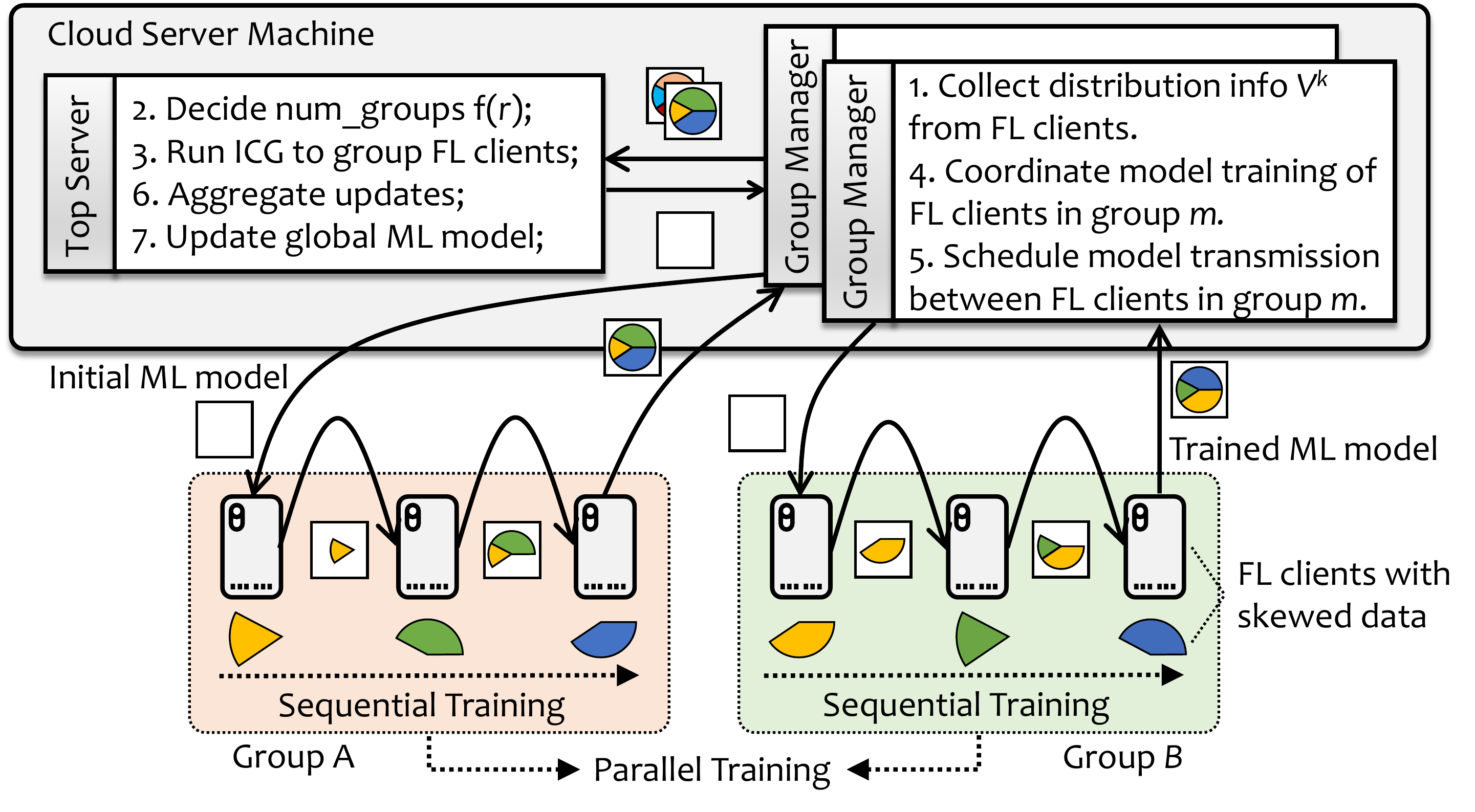}
\caption{An overview of the \NAME~framework.}
\label{fig:framework}
\end{figure}

\subsection{\GROUPER: The Inter-Cluster Grouping Algorithm\label{section:grouping}}
As required by STP, the equally sized groups containing heterogeneous FL clients should have similar overall class distributions. To achieve this goal, in this section, we first formalize the FL client grouping problem which is NP-hard, and then explain how to simplify to propose the \GROUPER~approach.

\vspace{2mm}
\noindent \textbf{(A) Problem Modeling}
\vspace{2mm}

\noindent Considering an $\mathcal{F}$-class classification task involving $K$ FL clients, STP needs to assign these clients to $M$ groups, where $M$ is determined by the group number growth function $f$ and the current round $r$. Our goal is to find a grouping strategy $\mathbf{x}\in\mathbb{I}^{M\times K}$ in the 0-1 space $\mathbb{I}=\{0,1\}$ to minimize the difference in class distributions of all groups, where $\mathbf{x}_m^k=1$ represents the device $k$ is assigned to the group $m$, $\mathcal{V}\in(\mathbb{Z^+})^{\mathcal{F}\times K}$ is the class distribution matrix composed of $\mathcal{F}$-dimensional class distribution vectors of $K$ FL clients, $\mathcal{V}_m\in(\mathbb{Z}^+)^{\mathcal{F}\times 1}$ represents the overall class distribution of group $m$, and $\left<\cdot,\cdot\right>$ represents the distance between two class distributions. The problem can be formalized as follows:
\begin{align}
\underset{\mathbf{x}}{\mathrm{minimize}}\qquad & z=\sum_{m_1=1}^{M-1}\sum_{m_2=m_1+1}^{M}<\mathcal{V}_{m_1},\mathcal{V}_{m_2}>, \label{eq:objective}\\
\mathrm{s.t.}\qquad & M=f(r), \label{eq:group-number}\\
& \sum_{k=1}^{K}\mathbf{x}_m^k\le \left\lceil \frac{K}{M} \right\rceil \quad \forall m=1,\cdots,M, \label{eq:group-capacity}\\
& \sum_{m=1}^{M}{\mathbf{x}_m^k}=1 \qquad\quad \forall k=1,\cdots,K, \label{eq:client-conflict}\\
& \mathcal{V}_m=\sum_{k=1}^{K}{\mathbf{x}_m^k\mathcal{V}^k}\quad \forall m=1,\cdots,M, \label{eq:overall-dist}\\
& \mathbf{x}_m^k\in\{0,1\}, ~k\in[1,K], ~m\in[1,M].
\label{eq:variable-constraint}
\end{align}
Constraint \eqref{eq:group-number} ensures that the number of groups $M$ meets $f(r)$ required by STP. Constraint \eqref{eq:group-capacity} ensures that the groups have similar or equal size $\left\lceil \frac{K}{M} \right\rceil$. Constraint \eqref{eq:client-conflict} ensures that each client can only be assigned to one group at a time. The overall class distribution $\mathcal{V}_m$ of the group $m$ is defined by Eq. \eqref{eq:overall-dist}, where $\mathcal{V}^k\in\mathcal{V}$ is the class distribution vector of client $k$. Constraint \eqref{eq:variable-constraint} restricts the decision variable $\mathbf{x}$ to only take up a value of 0 or 1.

\begin{proposition}\label{prop:bpp}
The NP-hard bin packing problem (BPP) can be reduced to the grouping problem in Eq. \eqref{eq:objective} to Eq. \eqref{eq:variable-constraint}, making it also an NP-hard problem.
\end{proposition}
\begin{proof}
The problem stated by Eq. \eqref{eq:objective} to  Eq. \eqref{eq:variable-constraint} is actually a BPP with additional constraints, where $K$ items with integer weight $\mathcal{V}^k$ and unit volume should be packed into the minimum number of bins of integer capacity $\left\lceil \frac{K}{M} \right\rceil$. The difference is that Eq. \eqref{eq:objective} to Eq. \eqref{eq:variable-constraint} restricts the number of available bins to $M$ instead of unlimited, and the difference in the bin weights not to exceed $\xi$. The input and output of BPP and Eq. \eqref{eq:objective} to Eq. \eqref{eq:variable-constraint} are matched, with only additional $\mathcal{O}(1)$ transformation complexity to set $M$ and $\xi$ to infinity. Therefore, BPP can call the solution of Eq. \eqref{eq:objective} to Eq. \eqref{eq:variable-constraint} in $\mathcal{O}(1)$ time to obtain its solution, which proves that the NP-hard BPP \cite{garey1978strong} can be reduced to the problem stated by Eq. \eqref{eq:objective} to Eq. \eqref{eq:variable-constraint}. Therefore, Eq. \eqref{eq:objective} to Eq. \eqref{eq:variable-constraint} is also an NP-hard problem.
\end{proof}

Therefore, it is almost impossible to find the optimal solution within a polynomial time. To address this issue, we adopt the centroid equivalence theorem to simplify the original problem to a constrained clustering problem.

\vspace{2mm}
\noindent \textbf{(B) Inter-Cluster Grouping (\GROUPER)}
\vspace{2mm}

Consider a constrained clustering problem with $K$ points and $L$ clusters, where the size of all clusters is strictly the same $K/L$.

\begin{assumption}\label{icg-assumption}
We make the following assumptions:
\begin{enumerate}
\item $K$ is divisible by $L$;
\item Take any point $\mathcal{V}^m_l$ from cluster $l$, the squared $l_2$-norm distance $\|\mathcal{V}^m_l-C_l\|_2^2$ between the point $\mathcal{V}^m_l$ and its cluster centroid $C_l$ is bounded by $\sigma_l^2$.
\item Take one point $\mathcal{V}^m_l$ from each of $L$ clusters at random, the sum of deviations of each point from its cluster centroid $\epsilon^m=\sum_{l=1}^{L}(\mathcal{V}^m_l-C_l)$ meets $\mathbf{E}\left[\epsilon^m\right]=0$.
\end{enumerate}
\end{assumption}

\begin{definition}[Group Centroid]\label{def:group-centroid}
Given $L$ clusters of equal size, let group $m$ be constructed from one point randomly sampled from each cluster $\{\mathcal{V}^m_1,\cdots,\mathcal{V}^m_L\}$. Then, the centroid of  group $m$ is defined as $C^m=\frac{1}{L}\sum_{l=1}^{L}\mathcal{V}^m_l$.
\end{definition}

\begin{proposition}\label{prop:centroid}
If Assumption \ref{icg-assumption} holds, suppose the centroid of cluster $l$ is $C_l=\frac{L}{K}\sum_{i=1}^{K/L}{\mathcal{V}_l^i}$ and the global centroid is $C_{\mathrm{global}}=\frac{1}{L}\sum_{l=1}^{L}{C_l}$. We have:
\begin{enumerate}
\item The group and global centroids are expected to coincide, $\mathbf{E}[C^m]=C_{\mathrm{global}}$.
\item The error $\|C^m-C_\mathrm{global}\|_2^2$ between the group and global centroids is bounded by $\frac{1}{L^2}\sum_{l=1}^{L}{\sigma_l^2}$.
\end{enumerate}
\end{proposition}
\begin{proof}
\begin{align}
\nonumber
\mathbf{E}[C^m]&=\mathbf{E}[\frac{1}{L}\sum_{l=1}^{L}\mathcal{V}^m_l] =\mathbf{E}[\frac{1}{L}\sum_{l=1}^{L}(\mathcal{V}^m_l-C_l+C_l)] \\
\nonumber
&=\mathbf{E}[\frac{1}{L}\sum_{l=1}^{L}(\mathcal{V}^m_l-C_l)+\frac{1}{L}\sum_{l=1}^{L}{C_l}]=\frac{1}{L}\mathbf{E}[\epsilon^m]+C_{\mathrm{global}}=C_{\mathrm{global}}, \\
\nonumber
\|C^m-C_\mathrm{global}\|_2^2&=\|\frac{1}{L}\sum_{l=1}^{L}{\mathcal{V}_l^m}-\frac{1}{L}\sum_{l=1}^{L}{C_l}\|_2^2=\frac{1}{L^2}\|\sum_{l=1}^{L}{(\mathcal{V}_l^m-C_l)}\|_2^2 \\
\nonumber
&\le\frac{1}{L^2}\sum_{l=1}^{L}\|\mathcal{V}_l^m-C_l\|_2^2=\frac{1}{L^2}\sum_{l=1}^{L}{\sigma_l^2}.
\end{align}
\end{proof}

Proposition \ref{prop:centroid} indicates that there exists a grouping strategy $\tilde{\mathbf{x}}$ and $\mathcal{V}_{m_1}=\sum_{k=1}^{K}\tilde{\mathbf{x}}_{m_1}^k\mathcal{V}^k=LC^{m_1}$, $\mathcal{V}_{m_2}=\sum_{k=1}^{K}\tilde{\mathbf{x}}_{m_2}^k\mathcal{V}^k=LC^{m_2}$ ($\forall m_1\ne m_2$), so that the objective in Eq. \eqref{eq:objective} turns to $z=\sum_{m_1\ne m_2}L<C^{m_1},C^{m_2}>$  and the expectation value reaches 0. This motivates us to use the constrained clustering model to solve $\tilde{\mathbf{x}}$ in the objective Eq. \eqref{eq:objective}. Therefore, we consider the constrained clustering problem below,
\begin{align}
\underset{\mathbf{y}}{\mathrm{minimize}}\qquad & \sum_{k=1}^{K}\sum_{l=1}^{L}\mathbf{y}^k_l\cdot\left(\frac{1}{2}\|\mathcal{V}^k-C_l\|_2^2\right), \label{eq:ccp-objective}\\
\mathrm{s.t.}\qquad & \sum_{k=1}^{K}\mathbf{y}_l^k=\frac{K}{L} \qquad \forall l=1,\cdots,L, \label{eq:ccp-least}\\
& \sum_{l=1}^{L}\mathbf{y}_l^k=1 \qquad\quad \forall k=1,\cdots,K, \label{eq:ccp-client}\\
& \mathbf{y}_l^k\in\{0,1\}, ~k\in[1,K], ~l\in[1,L], \label{eq:ccp-value}
\end{align}
where $\mathbf{y}\in\mathbb{I}^{L\times K}$ is a selector variable, $\mathbf{y}_l^k=1$ means that client $k$ is assigned to cluster $l$ while 0 means not, $C_l$ represents the centroid of cluster $l$. Eq. \eqref{eq:ccp-objective} is the standard clustering objective, which aims to assign $K$ clients to $L$ clusters so that the sum of the squared $l_2$-norm distance between the class distribution vector $\mathcal{V}^k$ and its nearest cluster centroid $C_l$ is minimized. Constraint \eqref{eq:ccp-least} ensures that each cluster has the same size $\frac{K}{L}$. Constraint \eqref{eq:ccp-client} ensures that each client can only be assigned to one cluster at a time. In this simplified problem, Constraint \eqref{eq:client-conflict} is relaxed to $\sum_{m=1}^{M}{\mathbf{x}_m^k}\le 1$ to satisfy the assumption that $K/L$ is divisible.

The above constrained clustering problem can be modeled as a minimum cost flow (MCF) problem and solved by network simplex algorithms \cite{bradley2000constrained}, such as \textsc{SimpleMinCostFlow} in Google OR-Tools. Then, we can alternately perform cluster assignment and cluster update to optimize $\mathbf{y}_l^k$ and $C_l (\forall k,l)$, respectively. Finally, we construct $M$ groups, each group consists of one client randomly sampled from each cluster without replacement, so that their group centroids are expected to coincide with the global centroid. The pseudo code is given in Algorithm \ref{alg:icg}. \GROUPER~has a complexity of $\mathcal{O}(\frac{K^6\mathcal{F}\tau}{M^2}\log{Kd})$, where $d=\max\{\sigma_l^2 | \forall l\in[1,L]\}$, and $K,M,\mathcal{F},\tau$ are the number of clients, groups, categories, and iterations, respectively. In our experiment, \GROUPER~is quite fast, and it can complete group assignment within only 0.1 seconds, with $K=364,M=52,\mathcal{F}=62$ and $\tau=10$.

\begin{algorithm}[t]
\caption{\colorbox{yellow}{\textsc{\textbf{Inter-Cluster-Grouping}}}}\label{alg:icg}
\begin{algorithmic}[1]
\Require All FL clients $\mathcal{C}$ (with attribute $\mathcal{V}^k$), the total number of FL clients $K$, the group number growth function $f$, the current training round $r$.
 \Ensure The grouping strategy $\mathcal{G}$.
\vspace{1mm}
\State{Randomly sample $L\cdot\lfloor\frac{K}{L}\rfloor$ clients from $\mathcal{C}$ to meet Assumption \ref{icg-assumption}, where $L=\lfloor\frac{K}{f(r)}\rfloor$;}
\Repeat
\State \textsc{Cluster Assignment:} Fix the cluster centroid $C_l$ and optimize $\mathbf{y}$ in Eq. \eqref{eq:ccp-objective} to Eq. \eqref{eq:ccp-value};
\State \textsc{Cluster Update:} Fix $\mathbf{y}$ and update the cluster centroid $C_l$ as follows, $$C_l\gets\frac{\sum_{k=1}^{K}{\mathbf{y}_l^k\mathcal{V}^k}}{\sum_{k=1}^{K}{\mathbf{y}_l^k}} \quad \forall l=1,\cdots,L;$$
\Until{$C_l$ converges;}
\State \textsc{Group Assignment:} Randomly sample one client from each cluster without replacement to construct group $\mathcal{G}_m (\forall m=1,\cdots,f(r))$;
\State \Return $\mathcal{G}=\{\mathcal{G}_1,\cdots,\mathcal{G}_{f(r)}\}$;
\end{algorithmic}
\end{algorithm}

%% file: experiment.tex
\section{Experimental Evaluation}\label{section:experiment}
 
\subsection{Experiment Setup and Evaluation Metrics}
\textbf{Environment and Hyperparameter Setup.} The experiment platform contains $K=368$ FL clients. The most commonly used FEMNIST\cite{caldas2018leaf} is selected as the benchmark dataset, which is specially designed for non-i.i.d. FL environment and is constructed by dividing 805,263 digit and character samples into 3,550 FL clients in a non-uniform class distribution, with an average of $n=226$ samples per client. For the resource-limited mobile devices, a lightweight neural network composed of 2 convolutional layers and 2 fully connected layers with a total of 6.3 million parameters is adopted as the training model. The standard mini-batch SGD is used by FL clients to train their local models, with the learning rate $\eta=0.01$, the batch size $b=5$ and the local epoch $e=1$. We test \NAME~for $R=500$ rounds. By default, we set the group sampling rate $\kappa=0.3$, the group number growth function $f=\textsc{Log}$ and the corresponding coefficients $\alpha=2$, $\beta=10$. The values of $\kappa$, $f$, $\alpha$, $\beta$ will be further tuned in the experiment to observe their performance influence.

\textbf{Benchmark Algorithms.} In order to highlight the effect of the proposed STP and \GROUPER~separately, we remove them from \NAME~to obtain the naive version, \NAIVE. Then, we compare the performance of the following versions of  \NAME~through ablation studies:
\begin{enumerate}
    \item \textit{\NAIVE}: FL clients are randomly assigned to a fixed number of groups, the clients in the group are trained in sequence and the groups are trained in parallel (e.g., Astraea\cite{duan2019astraea}).
    \item \textit{\NAIVE+\GROUPER}: The \GROUPER~grouping algorithm is adopted in \NAIVE~to assign FL clients to a fixed number of groups strategically. 
    \item \textit{\NAIVE+\GROUPER+STP} (\NAME): On the basis of \NAIVE+\GROUPER, FL clients are reassigned to a growing number of groups in each round as required by STP.
\end{enumerate}
In addition, seven state-of-the-art baselines are experimentally compared with \NAME. They are FedProx \cite{li2018federated}, FedMMD \cite{yao2018two}, FedFusion  \cite{yao2019towards}, IDA \cite{yeganeh2020inverse}, and FedAdagrad, FedAdam, FedYogi from \cite{reddi2020adaptive}.

\textbf{Evaluation Metrics.} In addition to the fundamental test accuracy and test loss, we also define the following metrics to assist in performance evaluation.

\textit{Class Probability Distance} (CPD). The maximum mean discrepancy (MMD) distance is a probability measure in the reproducing kernel Hilbert space. We define CPD as the kernel two-sample estimation with Gaussian radial basis kernel $\mathcal{K}$\cite{gretton2012kernel} to measure the difference in class probability (i.e., normalized class distribution) $\mathcal{P}=\mathrm{norm}(\mathcal{V}_{m_1}),\mathcal{Q}=\mathrm{norm}(\mathcal{V}_{m_2})$ between two groups $m_1,m_2$. Generally, the smaller the CPD, the smaller the data heterogeneity between two groups, and therefore the better the grouping strategy.
\begin{align}
\mathrm{CPD}(m_1,m_2)&=\mathrm{MMD}^2(\mathcal{P}, \mathcal{Q}) \\
\nonumber
&= \mathbf{E}_{x,x'\sim\mathcal{P}}\left[\mathcal{K}(x,x')\right]-2\mathbf{E}_{x\sim\mathcal{P},y\sim\mathcal{Q}}\left[\mathcal{K}(x,y)\right]+\mathbf{E}_{y,y'\sim\mathcal{Q}}\left[\mathcal{K}(y,y')\right].
\end{align}

\textit{Computational Time.} We define $T_{\mathrm{comp}}$ in Eq. \eqref{eq:comp-time-metric} to estimate the computational time cost, where the number of floating point operations (FLOPs) is $\mathcal{N}_{\mathrm{calc}}=96$M FLOPs per sample and $\mathcal{N}_{\mathrm{aggr}}=6.3$M FLOPs for global aggregation, and $\mathcal{T}_{\mathrm{FLOPS}}=567$G FLOPs per second is the computing throughput of the Qualcomm Snapdragon 835 smartphone chip equipped with Adreno 540 GPU.
\begin{equation}
\label{eq:comp-time-metric}
T_{\mathrm{comp}}(R)=\sum_{r=1}^{R}\left(\underbrace{\frac{\mathcal{N}_{\mathrm{calc}}}{\mathcal{T}_{\mathrm{FLOPS}}}\cdot\frac{neK}{\min\left\{K,f(r)\right\}}}_{\mathrm{Local~Training}}+\underbrace{\frac{\mathcal{N}_{\mathrm{aggr}}}{\mathcal{T}_{\mathrm{FLOPS}}}\cdot \left[\kappa f(r)-1\right]}_{\mathrm{Global~Aggregation}}\right)~(\mathrm{s}).
\end{equation}

\textit{Communication Time and Traffic.} We define $T_{\mathrm{comm}}$ in Eq. \eqref{eq:comm-time-metric} to estimate the communication time cost and $D_{\mathrm{comm}}$ in Eq. \eqref{eq:comm-traffic-metric} to estimate the total traffic, where the FL model size is $\mathcal{M}=25.2$MB, the inbound and outbound transmission rates are $\mathcal{R}_{\mathrm{in}}=\mathcal{R}_{\mathrm{out}}=567\mathrm{Mbps}$ (tested in the Internet by AWS EC2 r4.large 2 vCPUs with disabled enhanced networking). Eq. \eqref{eq:comm-time-metric} to Eq. \eqref{eq:comm-traffic-metric} consider only the cross-WAN traffic between FL clients and group managers, but the traffic between the top server and group managers is ignored because they are deployed in the same physical machine.
\begin{equation}
\label{eq:comm-time-metric}
T_{\mathrm{comm}}(R)=8\kappa K\mathcal{M}R(\frac{1}{\mathcal{R}_{\mathrm{in}}}+\frac{1}{\mathcal{R}_{\mathrm{out}}})~(\mathrm{s}),
\end{equation}
\begin{equation}
\label{eq:comm-traffic-metric}
D_{\mathrm{comm}}(R)=2\kappa K\mathcal{M}R~(\mathrm{Bytes}).
\end{equation}

Please note that Eq. \eqref{eq:comp-time-metric} to Eq. \eqref{eq:comm-time-metric} are theoretical metrics, which do not consider memory I/O cost, network congestion, and platform configurations such as different versions of CUDNN/MKLDNN libraries.

\subsection{Results and Discussion}
\textbf{The effect of \GROUPER~and STP.} We first compare the CPD of FedAvg\cite{mcmahan2017communication}, \NAIVE~and \NAIVE+\GROUPER~in Figure  \ref{fig:probability-distance}. These CPDs are calculated between every pair of FL clients. The results show that \NAIVE+\GROUPER~reduces the median CPD of FedAvg by $82\%$ and \NAIVE~by $41\%$. We also show their accuracy performance in Figure \ref{fig:acc-curve}. The baseline \NAIVE~quickly converges but only achieves the accuracy similar to FedAvg. Instead, \NAIVE+\GROUPER~improves the accuracy by $6\%$. This shows that reducing the data heterogeneity among groups can indeed effectively improve FL performance in the presence of non-i.i.d. data. Although \NAIVE+\GROUPER~is already very effective, it still has defects. Figure \ref{fig:loss-curve} shows a rise in the loss value of \NAIVE+\GROUPER, which indicates that it has been overfitted. That is because the training mode of \NAIVE+\GROUPER~is static, it may learn the client order and forget the previous data. Instead, the dynamic \NAME~overcomes overfitting and eventually converges to a higher accuracy $85.4\%$, which proves the effectiveness of combining STP and \GROUPER.

\begin{figure}[t]
\centering
\subfloat[CPD]{\includegraphics[width=.33\textwidth]{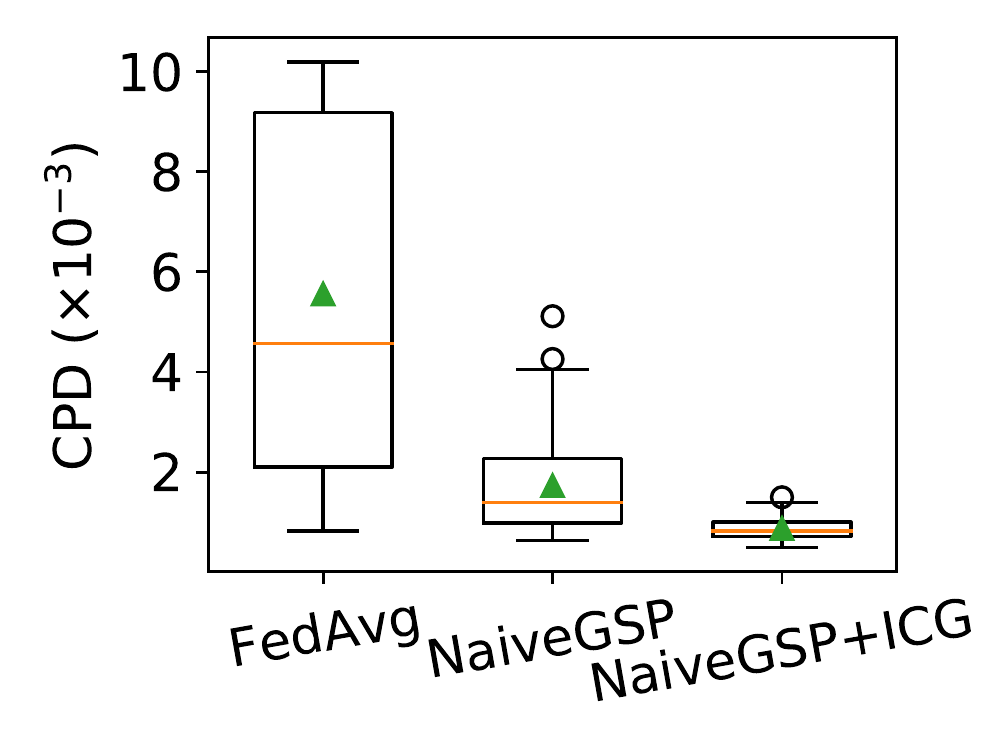}
\label{fig:probability-distance}}
\subfloat[Accuracy curve]{\includegraphics[width=.33\textwidth]{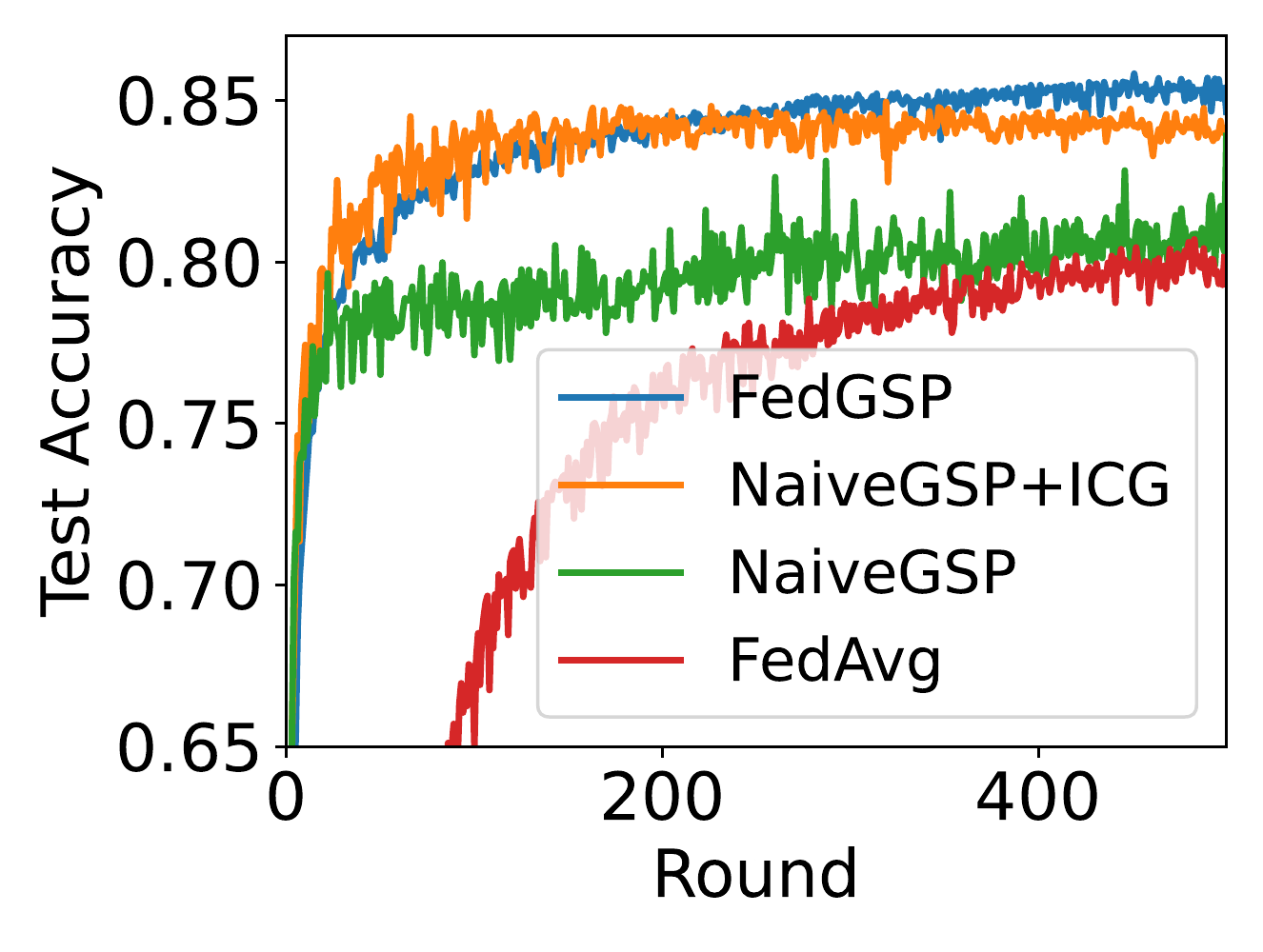}
\label{fig:acc-curve}}
\subfloat[Loss curve]{\includegraphics[width=0.34\textwidth]{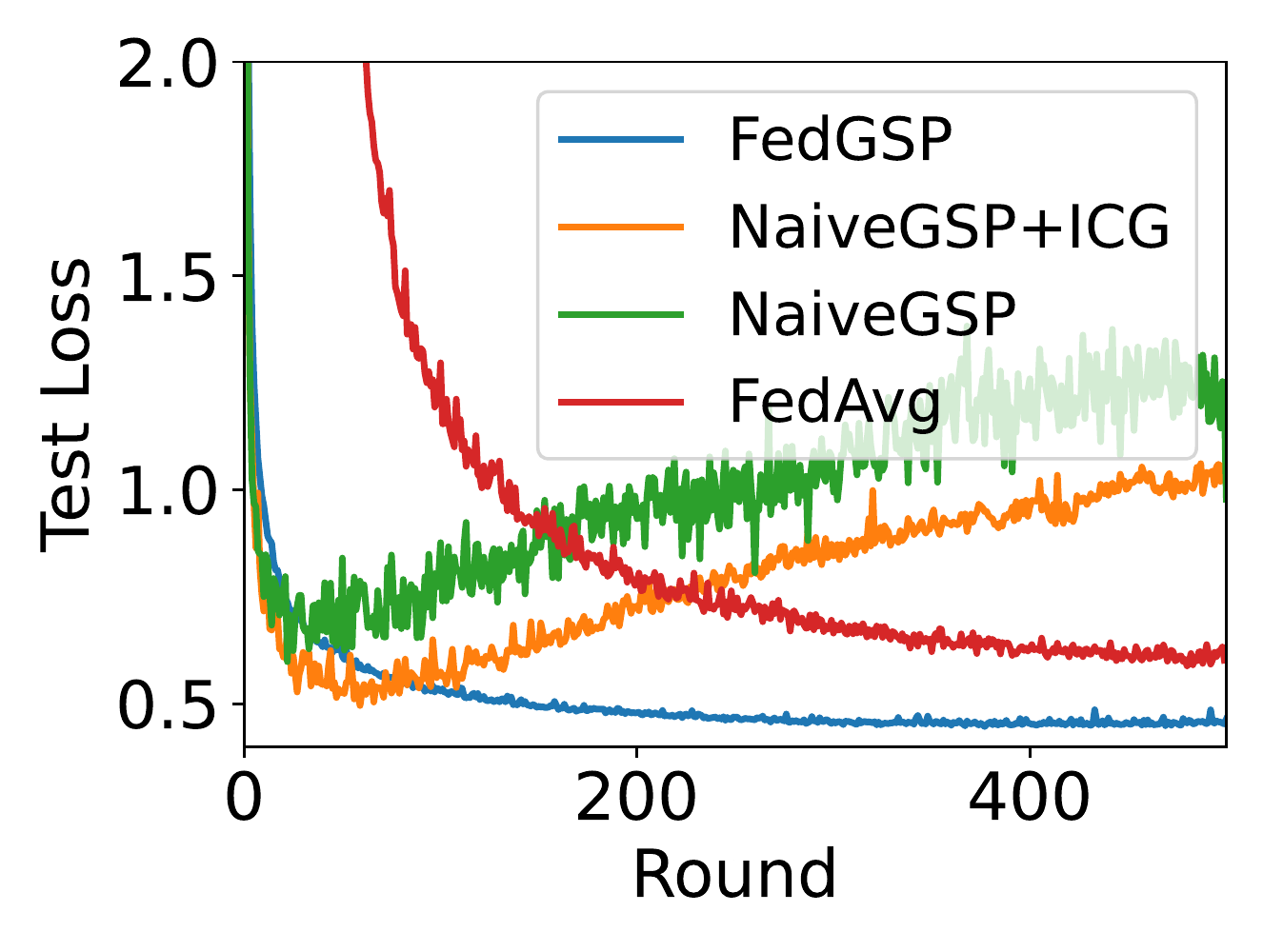}
\label{fig:loss-curve}}
\caption{Comparison among FedAvg, \NAIVE, \NAIVE+\GROUPER~and \NAME~in (a) CPD, (b) accuracy curve and (c) loss curve. In subfigure (a), the orange line represents the median value and the green triangle represents the mean value.}
\end{figure}

\textbf{The effect of the growth function $f$ and its coefficients $\alpha,\beta$.} To explore the performance influence of different group number growth functions $f$, we conduct a grid search on $f=\{\textsc{Linear},\textsc{Log},\textsc{Exp}\}$ and $\alpha,\beta$. The test loss heatmap is shown in Figure \ref{fig:loss-heatmap}. The results show that the logarithmic growth function achieves smaller loss 0.453 with $\alpha=2,\beta=10$ among 3 candidate functions. Besides, we found that both lower and higher $\alpha,\beta$ lead to higher loss values. The reasons may be that a slow increase in the number of groups leads to more STM and results in overfitting, while a rapid increase in the number of groups makes \NAME~degenerate into FedAvg prematurely and suffers the damage of data heterogeneity. Therefore, we recommend $\alpha\cdot\beta$ to be a moderate value, as shown in the green area.

\begin{figure}[t]
\centering
\subfloat[Linear]{\includegraphics[width=0.34\textwidth]{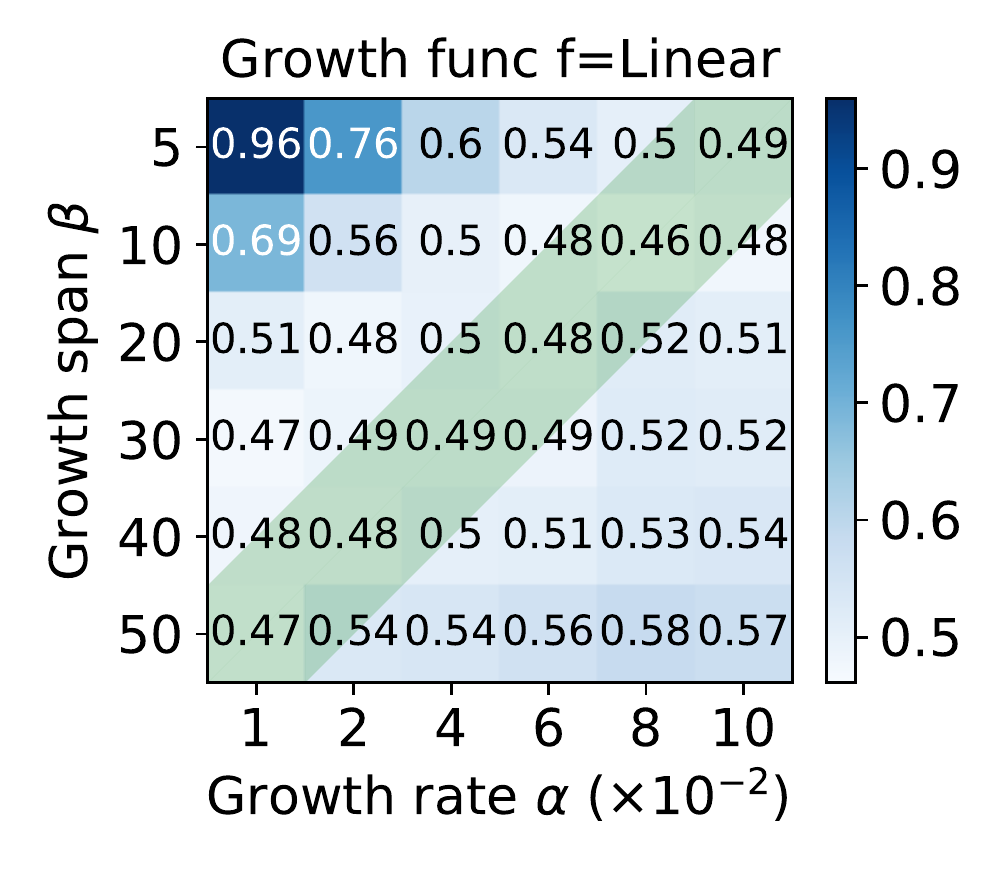}}
\subfloat[Logarithmic]{\includegraphics[width=0.34\textwidth]{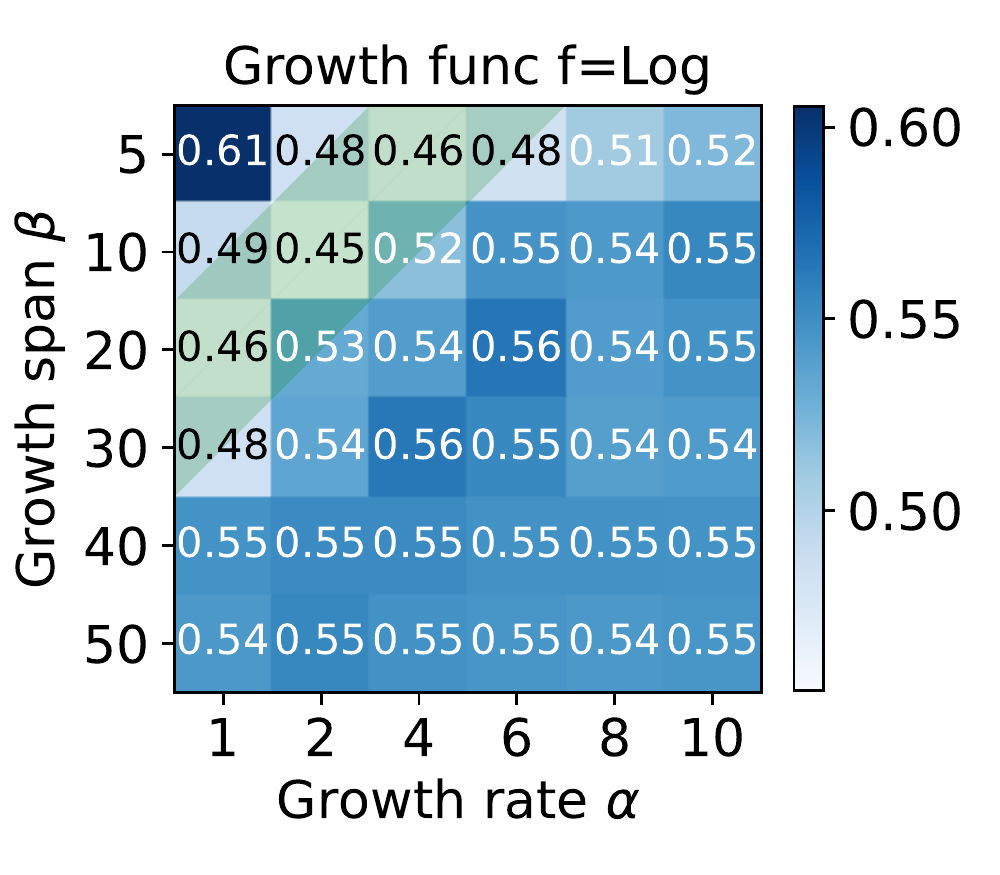}}
\subfloat[Exponential]{\includegraphics[width=0.34\textwidth]{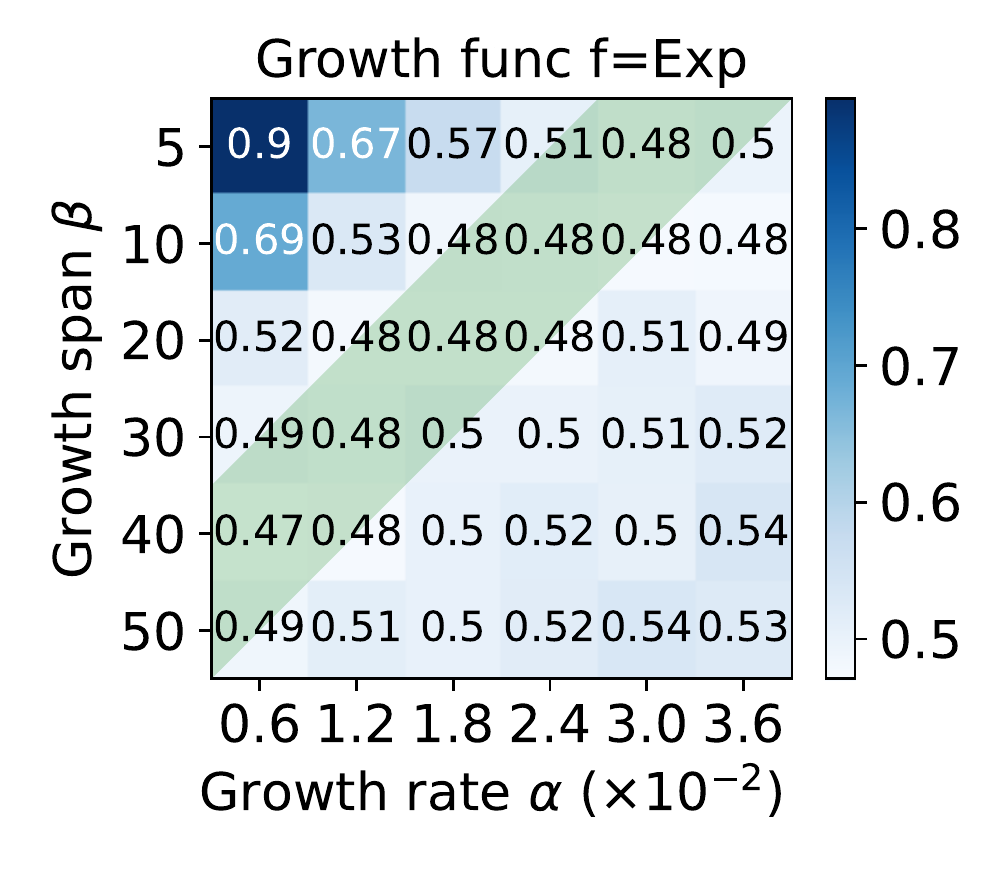}}\\
\caption{Test loss heatmap of (a) linear, (b) logarithmic and (c) exponential growth functions over different $\alpha$ and $\beta$ settings in \NAME.}
\label{fig:loss-heatmap}
\end{figure}

\begin{figure}[t]
\centering
\subfloat[Accuracy]{\includegraphics[width=0.33\textwidth]{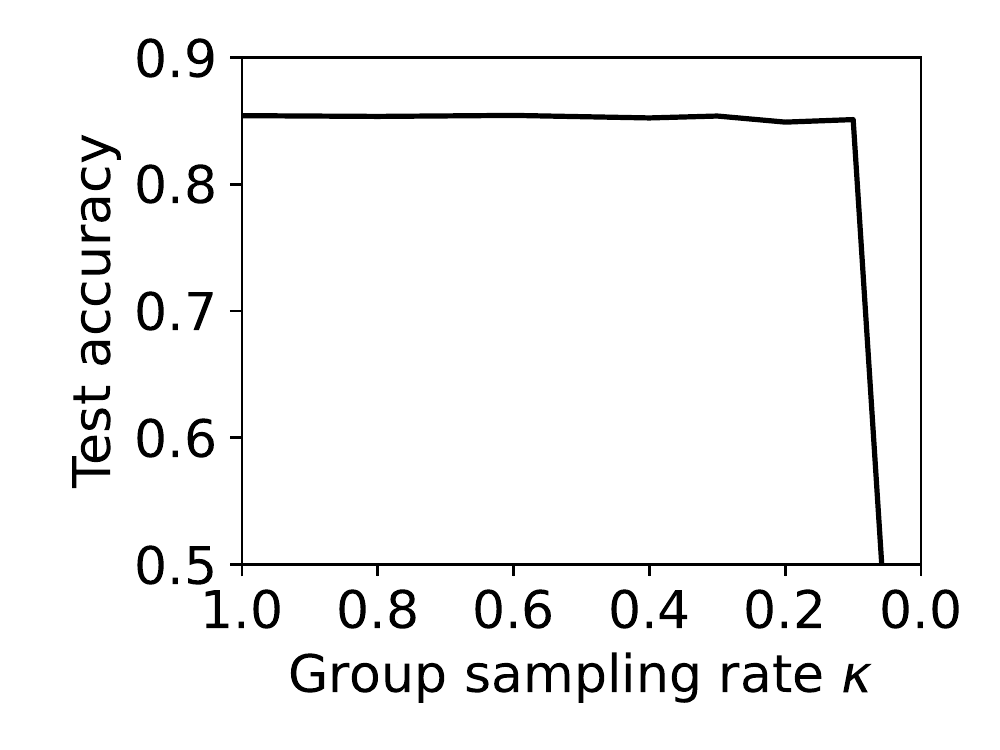}
\label{fig:acc-vs-kappa}}
\subfloat[Computational Time]{\includegraphics[width=0.33\textwidth]{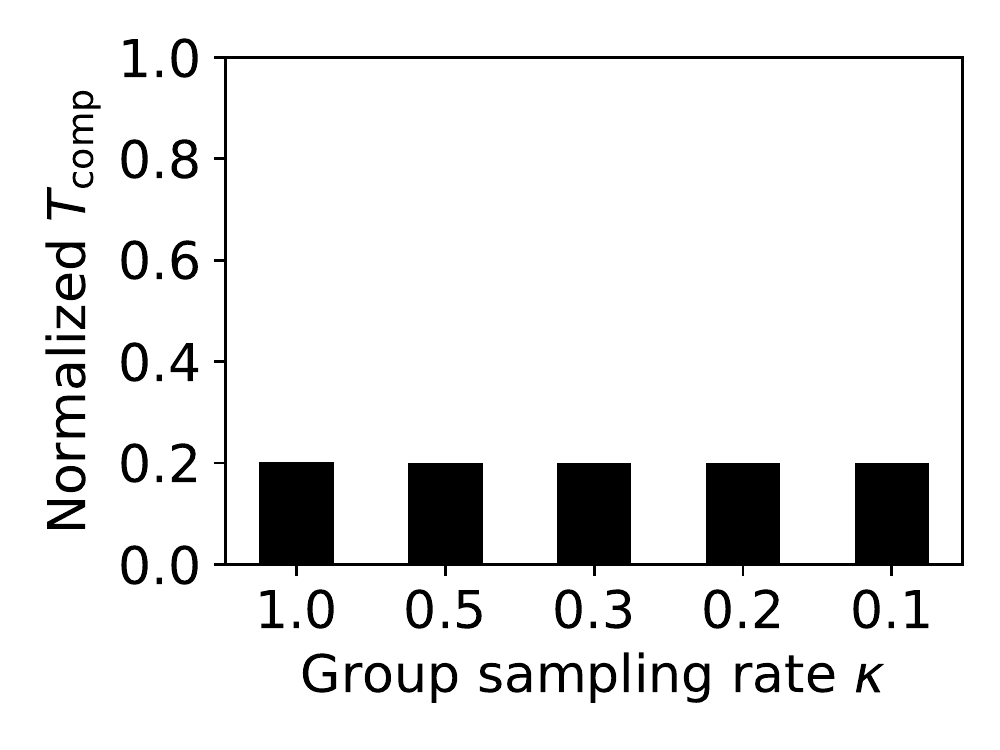}
\label{fig:comp-time-vs-kappa}}
\subfloat[Communication Time]{\includegraphics[width=0.33\textwidth]{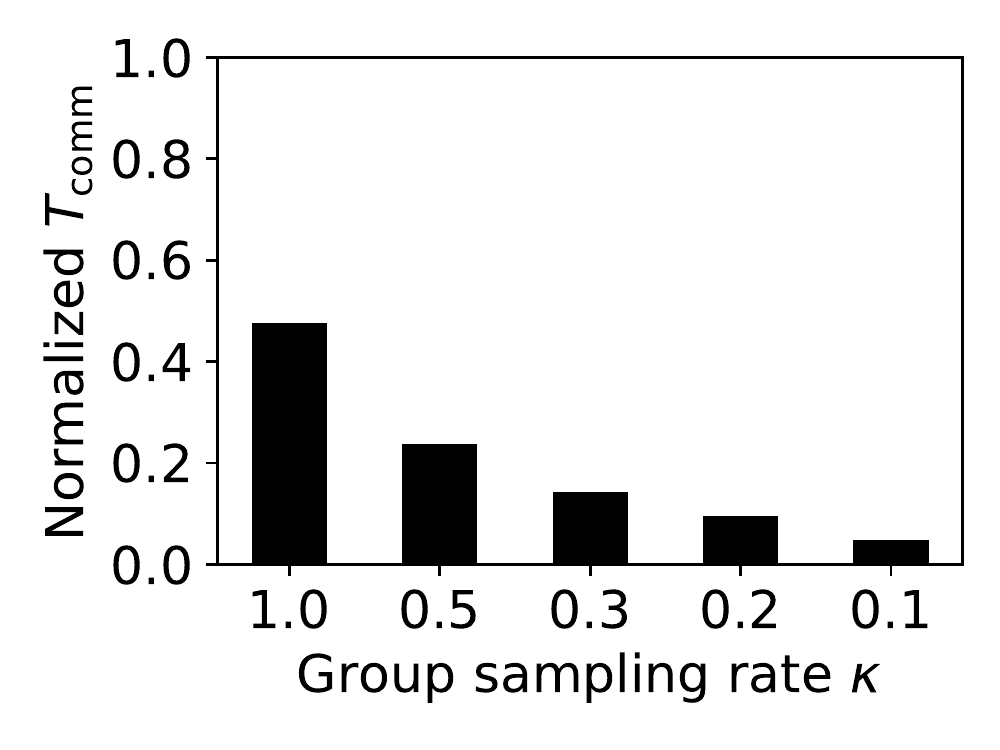}
\label{fig:comm-time-vs-kappa}}
\caption{Comparison of (a) accuracy and the normalized (b) computational time and (c) communication time over different $\kappa$ settings in \NAME.}
\label{fig:kappa-comparison}
\end{figure}

\textbf{The effect of the group sampling rate $\kappa$.} $\kappa$ controls the participation rate of groups (also the participation rate of FL clients) in each round. We set $\kappa=\{0.1,0.2,0.3,0.5,1.0\}$ to observe its effect on accuracy and time cost. Figure \ref{fig:acc-vs-kappa} shows the robustness of accuracy to different values of $\kappa$. This is expected because \GROUPER~forces the data of each group to become homogeneous, which enables each group to individually represent the global data. In addition, Figures \ref{fig:comp-time-vs-kappa} and \ref{fig:comm-time-vs-kappa} show that $\kappa$ has a negligible effect on computational time $T_{\mathrm{comp}}$, but a proportional effect on communication time $T_{\mathrm{comm}}$ because a larger $\kappa$ means more model data are involved in data transmission. Therefore, we recommend that only $\kappa\in[0.1,0.3]$ of groups are sampled to participate in FL in each round to reduce the overall time cost. In our experiments, we set $\kappa=0.3$ by default.
  
\textbf{The performance comparison of \NAME.} We compare \NAME~with seven state-of-the-art approaches and summarize their test accuracy, test loss and training rounds (required to reach the accuracy of $80\%$) in Table \ref{table:fedgsp-vs-others}. The results show that \NAME~achieves $5.3\%$ higher accuracy than FedAvg and reaches the accuracy of $80\%$ within only $34$ rounds. Moreover, \NAME~outperforms all the comparison approaches, with an average of $3.7\%$ higher accuracy, $0.123$ lower loss and $84\%$ less rounds, which shows its effectiveness to improve FL performance in the presence of non-i.i.d. data.

\textbf{The time and traffic cost of \NAME.} Figure \ref{fig:6} visualizes the time cost and total traffic of \NAME~and FedAvg when they reach the accuracy of $80\%$. The time cost consists of computational time $T_{\mathrm{comp}}(R)$ and communication time $T_{\mathrm{comm}}(R)$, of which $T_{\mathrm{comm}}(R)$ accounts for the majority due to the huge data traffic from hundreds of FL clients has exacerbated the bandwidth bottleneck of the cloud server. Figure \ref{fig:6} also shows that \NAME~spends $93\%$ less time and traffic than FedAvg, which benefits from a cliff-like reduction in the number of training rounds $R$ (only 34 rounds to reach the accuracy of $80\%$). Therefore, \NAME~is not only accurate, but also training- and communication-efficient.

\vspace{10mm}
\begin{minipage}{\textwidth}
  \begin{minipage}[b]{0.4\textwidth}
    \centering
    \includegraphics[width=\textwidth]{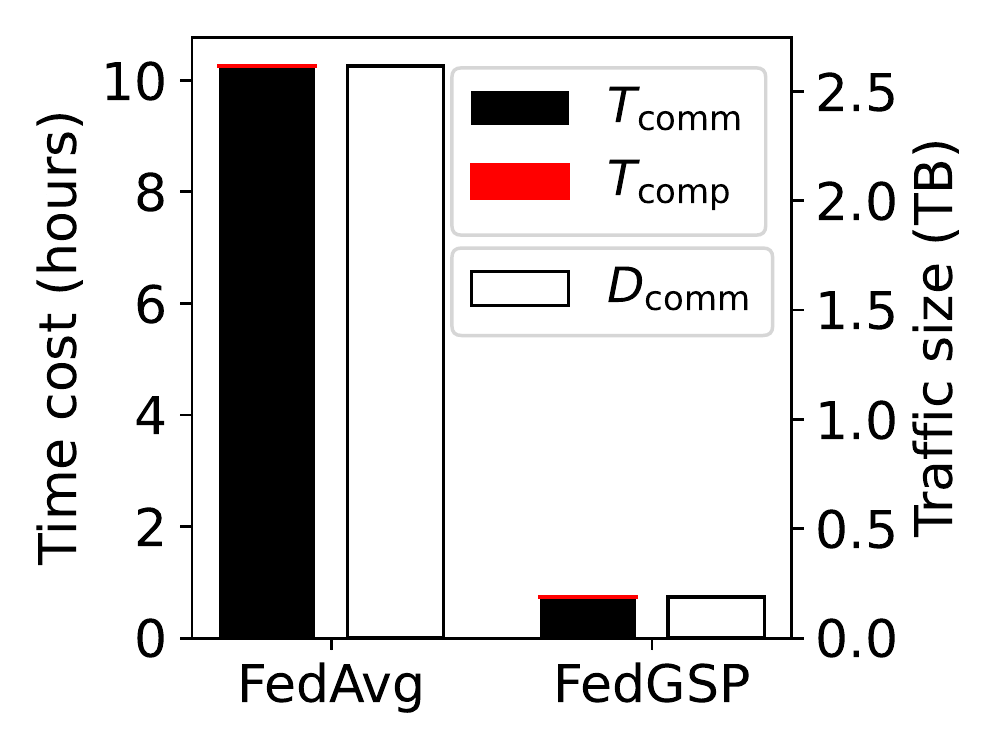}\label{fig:cost-compare}
    \captionof{figure}{Comparison of time and traffic cost to reach $80\%$ accuracy.}\label{fig:6}
  \end{minipage}
  \quad
  \begin{minipage}[b]{0.5\textwidth}
    \centering
    \scalebox{0.9}{
    \begin{tabular}{c|c|c|c}
      \hline
      Algorithm & Accuracy & Loss & Rounds \\ 
      \hline\hline
      FedAvg & 80.1\% & 0.602 & 470 \\ 
      \hline
      FedProx & 78.7\% & 0.633 & $\times$ \\ 
      FedMMD & 81.7\% & 0.587 & 336 \\
      FedFusion & 82.4\% & 0.554 & 230 \\
      IDA & 82.0\% & 0.567 & 256 \\
      FedAdagrad & 81.9\% & 0.582 & 297 \\
      FedAdam & 82.1\% & 0.566 & 87 \\
      FedYogi & 83.2\% & 0.543 & 93 \\ \hline
      \textbf{\NAME} & \textbf{85.4\%} & \textbf{0.453} & \textbf{34} \\ 
      \hline
    \end{tabular}}
      \captionof{table}{Comparison of accuracy, loss, and rounds required to reach $80\%$ accuracy.}
      \label{table:fedgsp-vs-others}
    \end{minipage}
\end{minipage}


%% file: conclusion.tex
\section{Conclusions}
In this paper, we addressed the problem of FL model performance degradation in the presence of non-i.i.d. data. We proposed a new concept of dynamic STP collaborative training that is robust against data heterogeneity, and a grouped framework \NAME~to support dynamic management of the continuously growing client groups. In addition, we proposed \GROUPER~to support efficient group assignment in STP by solving a constrained clustering problem with equal group size constraint, aiming to minimize the data distribution divergence among groups. We experimentally evaluated \NAME~on LEAF, a widely adopted FL benchmark platform, with the non-i.i.d. FEMNIST dataset. The results showed that \NAME~significantly outperforms seven state-of-the-art approaches in terms of model accuracy and convergence speed. In addition, \NAME~is both training- and communication-efficient, making it suitable for practical applications. 